%% file: main.tex
\pdfoutput=1
\documentclass[10pt]{article}

\usepackage[T1]{fontenc}

\usepackage{stmaryrd}
\usepackage{amsthm}
\usepackage{amsmath}
\usepackage{amssymb}
\usepackage{mathrsfs}
\usepackage{soul}
\usepackage{pgf,tikz}
\usepackage{graphicx}
\usepackage{xcolor}
\usetikzlibrary{arrows}
\usepackage{listings}
\usepackage{hyperref}
\hypersetup{
    colorlinks = true,
    citecolor = blue,
}
\usepackage{cases}
\usepackage{mathabx}
\usepackage{algorithm}
\usepackage{algorithmic}
\usepackage[nottoc,numbib]{tocbibind}
\usepackage{natbib}
\usepackage{microtype}
\usepackage{graphicx}
\usepackage{subfigure}
\usepackage{booktabs}
\usepackage{hyperref}

 \bibpunct{(}{)}{;}{a}{,}{,}
\usepackage[mathcal]{eucal}

\theoremstyle{plain}
\newtheorem{thm}{Theorem}[section]
\newtheorem{lem}[thm]{Lemma}
\newtheorem{prop}[thm]{Proposition}

\theoremstyle{definition}

\theoremstyle{remark}
\newtheorem{rem}{Remark}

\DeclareMathOperator{\OT}{OT}
\DeclareMathOperator{\KL}{KL}
\newcommand{\R}{\mathbb{R}}

\newcommand{\X}{\mathcal{X}}
\newcommand{\Y}{\mathcal{Y}}

\newcommand{\Esp}[2]{\mathbb{E}_{#1}\left[ #2 \right]}

\newcommand{\vertiii}[1]{{\left\vert\kern-0.25ex\left\vert\kern-0.25ex\left\vert #1 
\right\vert\kern-0.25ex\right\vert\kern-0.25ex\right\vert}}

\hypersetup{urlcolor=blue, colorlinks=true}

\title{Stochastic Optimization for\\ Regularized Wasserstein Estimators}
\author{	Marin Ballu\thanks{
  Department of Pure Mathematics and Mathematical Statistics, University of Cambridge,
	\texttt{mb2193@cam.ac.uk} }
	\and
  Quentin Berthet\thanks{
  Google Research, Brain team, Paris,
  \texttt{qberthet@google.com}
  }\and Francis Bach\thanks{
  INRIA - DI, ENS,
PSL Research University Paris,
  \texttt{francis.bach@inria.fr} }
}
\date{}

\begin{document}

\maketitle
\input{Sections/0-abstract}
\input{Sections/1-introduction}
\input{Sections/2-pbdescription}
\input{Sections/3-dualformulations}
\input{Sections/4-stochoptmethods}
\input{Sections/5-results}

\input{Sections/6-conclusion}
\bibliography{references}
\input{Sections/A-appendix}

\end{document}

%% file: Sections/0-abstract.tex
\begin{abstract}
Optimal transport is a foundational problem in optimization, that allows to compare probability distributions while taking into account geometric aspects. Its optimal objective value, the Wasserstein distance, provides an important loss between distributions that has been used in many applications throughout machine learning and statistics. Recent algorithmic progress on this problem and its regularized versions have made these tools increasingly popular. However, existing techniques require solving an optimization problem to obtain a single gradient of the loss, thus slowing down first-order methods to minimize the sum of losses, that require many such gradient computations. In this work, we introduce an algorithm to solve a regularized version of this problem of Wasserstein estimators, with a time per step which is sublinear in the natural dimensions of the problem. We introduce a dual formulation, and optimize it with stochastic gradient steps that can be computed directly from samples, without solving additional optimization problems at each step. Doing so, the estimation and computation tasks are performed jointly. We show that this algorithm can be extended to other tasks, including estimation of Wasserstein barycenters. We provide theoretical guarantees and illustrate the performance of our algorithm with experiments on synthetic data.
\end{abstract}

%% file: Sections/1-introduction.tex
\section{Introduction}

Optimal transport is one of the foundational problems of  {optimisation} \citep{monge1781memoire,kantorovich2006translocation}, and an important topic in {analysis} \citep{villani2008optimal}. It asks how one can transport mass with distribution measure $\mu$ to  another distribution measure $\nu$, with minimal global transport cost. It can also be written with a probabilistic interpretation, known as the Monge-Kantorovich formulation, of finding a joint distribution $\pi$ in the set $\Pi(\mu,\nu)$ of those with marginals $\mu$ and $\nu$, minimizing an expected cost between variables $X$ and $Y$. The minimum value gives rise to a natural statistical tool to compare distributions, known as the Wasserstein (or earth-mover's) distance,
\begin{equation*}
W_c(\mu,\nu) = \OT(\mu,\nu) = \min_{\pi\in\Pi(\mu,\nu)}\Esp{(X,Y)\sim \pi}{ c(X,Y) } .
\end{equation*}
In the case of finitely supported measures, taken with same support size $n$ for ease of notation, such as two empirical measures from samples, it is written as a linear program (on the right). It can be solved by the Hungarian algorithm \citep{kuhn1955hungarian}, which runs in time $O(n^3)$. While tractable, this is still relatively expensive for extremely large-scale applications in modern machine learning, where one hopes for running times that are linear in the size of the input (here $n^2$).

Attention to this problem has been recently renewed  in machine learning, in particular due to recent advances to efficiently solve an entropic-regularized version \citep{cuturi2013sinkhorn}, and its uses in many applications~\citep[see e.g.][for a survey]{peyre2019computational}, as it allows to capture the geometric aspects of the data. This problem has a strongly convex objective, and its solution converges to that of the optimal transport problem whenthe regularization parameter goes to 0. It can be easily solved with the Sinkhorn algorithm \citep{sinkhorn1964relationship,altschuler2017near}, or by other methods in time $O(n^2 \log n)$ \citep{dvurechensky2018computational}.

These tools have been applied in a wide variety of fields, from machine learning \citep{alvarez2018structured,arjovsky2017wasserstein,gordaliza2019obtaining,flamary2018wasserstein}, natural language processing \citep{grave2019unsupervised,alaux2018unsupervised,alvarez2018structured}, computer graphics \citep{feydy2017optimal,lavenant2018dynamical, solomon2015convolutional}, the natural sciences \citep{del2019optimalflow,schiebinger2019optimal}, and learning under privacy \citep{boursier2019private}.

Of particular interests to statistics and machine learning are analyses of this problem with only sample access to the distributions. There have been growing efforts to estimate either the objective value of this problem, or the unknown distribution, with this metric or associated regularized metrics (see below) \citep{weed2019sharp,genevay2019sample,uppal2019nonparametric}.  One of the motivations are variational Wasserstein problems, where the objective value of an optimal transport problem is used as a loss, and one seeks to minimize in a parameter $\theta$ an objective that depends on a known distribution $\nu_\theta$
\begin{equation*}
    \min_{\theta \in \Theta} \OT(\nu_\theta, \mu)\, ,
\end{equation*}
where $\mu$ is only accessible through samples. This method for estimation, referred to as {\em minimum Kantorovich estimators} \citep{bassetti2006minimum}, mirrors the interpretation of likelihood maximization as the minimization of $\KL(\nu_\theta,\mu)$, with the Kullback-Leibler divergence.

The value of the entropic-regularized problem, or of the related {\em Sinkhorn divergence}, can also be used as a loss in learning tasks \citep{alvarez2018structured,genevay2017learning,luise2018differential}, and compared to other metrics such as maximum mean discrepency \citep{gretton2012kernel,feydy2019interpolating,arbel2019maximum}. One of the advantages of the regularized problem is the existence of gradients in the parameters of the problem (cost matrix, target measures). 

The problem of minimizing this loss for the $\ell_2$ cost over $\R^d$ has been shown to be equivalent to maximum likelihood Gaussian deconvolution \citep{rigollet2018entropic}. We show here that this result can be generalized for all cost functions to maximum likelihood estimation for a kernel inversion problem. It is not only the solution of a stochastic optimization problem, but also an estimator, referred to here as the {\em regularized Wasserstein estimator}.

In this work, we propose a new stochastic optimization scheme to minimize the $\OT_\varepsilon$ between an unknown discrete measure $\mu$ and another discrete measure $\nu \in \mathcal{M}$, with an additional regularization term on $\nu$. There are many connections between this problem and stochastic optimization: by a dual formulation, the value $\OT_\varepsilon(\mu,\nu)$ can be written as the optimum of an expectation in $\mu,\nu$, allowing simple computations with only sample access \citep{genevay2016stochastic}. Here, we take this one step further and design an algorithm to {\em optimize} in $\nu$, not just evaluate this loss. A direct approach is to optimize by first-order methods, by the use of stochastic gradients in $\nu$ at each step \citep{genevay2017learning}. However, these gradient estimates are based on dual solutions of the regularized problem, so obtaining them requires to solve an optimization problem, with running time scaling quadratically in the intrinsic dimension of the problem (the size of the supports of $\mu,\nu$). For the dual formulation that we introduce, stochastic gradients can be directly computed from samples. Algorithmic techniques exploiting the particular structure of the dual formulation for this regularization allow us to compute these gradients in constant time. We follow here the recent developments in \emph{sublinear algorithms} based on stochastic methods~\citep{clarkson2012sublinear}.

We provide theoretical guarantees on the convergence of the final iterate $\nu_t$ to the true minimizer $\nu^*$, and demonstrate these results on simulated experiments.

%% file: Sections/2-pbdescription.tex
\section{Problem Description}

\paragraph{Definitions.}
Let $\mu$ be a probability measure on $\R^d$ with finite support $\X=\{x_i\}_{1\leq i\leq I}\subset\R^d$ and a family $\mathcal{M}$ of probability measures. The measures in $\mathcal{M}$ should all be absolutely continuous with respect to a known measure~$\beta$ supported in the finite set $\Y = \{y_j\}_{1\leq j\leq J}\subset \R^d$. We consider the following minimization problem:
\begin{equation}\label{problem}
    \min_{\nu\in \mathcal{M}}\OT_\varepsilon(\mu, \nu) + \eta \KL(\nu, \beta). 
\end{equation}
In this expression, $\OT_\varepsilon$ is the regularised optimal transport cost defined by the following expression
\begin{equation}\label{OTproblem}
    \OT_\varepsilon(\mu,\nu) = \min_{\pi\in\Pi(\mu,\nu)}\Esp{(X,Y)\sim \pi}{ c(X,Y) }+ \varepsilon \KL(\pi, \mu\otimes\nu),
\end{equation}
where the minimum is taken over the set 
\begin{equation*}
\Pi(\mu,\nu) = \left\{\pi\in\mathcal{P}(X\times Y): \pi_X = \mu,\ \pi_Y = \nu\right\}
\end{equation*}
of couplings of $\mu$ and $\nu$,
and  $c$ is a cost function in $\R^d$. The operator $\KL(\cdot,\cdot)$ is the Kullback-Leibler divergence, defined as
\begin{equation*}
\KL(\mu_1,\mu_2) = \Esp{Z \sim\mu_2}{\frac{d\mu_1}{d\mu_2}(Z)\log\left(\frac{d\mu_1}{d\mu_2}(Z)\right)},
\end{equation*}
for two measures $\mu_1$ and $\mu_2$ such that $\mu_1\ll\mu_2$. We assume that $\mathcal{M}$ is convex for the problem to be a convex optimization problem, and compact to guarantee that the minimum is attained.

\begin{rem} If $c$ is a distance and if $\varepsilon = \eta = 0$, then $\OT_\varepsilon$ is a Wasserstein distance and our problem can be seen as computing a projection of $\mu$ onto $\mathcal{M}$. In the discrete case, the solution to the unregularized problem is the distribution $\nu$ such that $\nu(y) = \mu(x)$, where $y$ is the nearest neighbour in $\mathcal{Y}$ of $x$.
\end{rem}

\paragraph{Learning problem.} Our objective is to solve the optimization problem in Equation~(\ref{problem}), given observations $X_i$ independent and identically distributed (i.i.d.) from $\mu$ that is unknown, and sample access to $\beta$. These can be assumed to be simulated by the user if $\beta$ is known, as part of the regularization. This problem can be either be interpreted as an unsupervised learning problem or as estimation in an inverse problem, and we refer to it as {\em regularized Wasserstein estimation}. The term in Kullback-Leibler (or entropy, up to an offset) are classical manners in which a probability can be regularized.

\paragraph{Maximum likelihood interpretation.}
While the unregularized problem has a trivial solution, there is in general no closed form for positive $\varepsilon$. When $\varepsilon>0$, $\eta = 0$ and $\mathcal{M}$ is the set of all probability measures on $Y$, then our problem is equivalent to the maximum likelihood estimator for a kernel inversion problem. This corresponds to estimating the unknown initial distribution of a random variable $Y$, but only by observing it {\em after} the action of a specific transition kernel $\kappa$ \citep[see, e.g.,][for the statistical complexity of estimating initial dustributions under general Markov kernels]{berthet2019statistical}.
\begin{prop}[MLE interpretation] Let $\mathcal{M}$ be the set of all probability measures on $Y$, let $\nu^*$ be a measure on $Y$, and let $\kappa : Y\to X$ be a transition kernel of the form
\begin{equation*}
\kappa(x, y) = \frac{\exp{\left(-\frac{c(x,y)}{\varepsilon}\right)}}{\sum_{x' \in X}\exp{\left(-\frac{c(x',y)}{\varepsilon}\right)}},
\end{equation*}
the observed measure is $\mu = \kappa \nu^*$, which can be written as
\begin{equation*}
    \mu(x) = \int_Y \kappa(x, y)d\nu^*(y).
\end{equation*}
The maximum likelihood estimation of $\nu^*$ for this observation is 
\begin{equation*}
    \hat{\nu} := \arg\max_{\nu\in\mathcal{M}}\sum_i \log (\kappa \nu)(X_i).
\end{equation*}
This estimator also verifies
\begin{equation}\label{nu=0}
    \hat{\nu} = \arg\min_{\nu\in \mathcal{M}}\OT_\varepsilon(\mu, \nu). 
\end{equation}
\end{prop}
\begin{rem}
If $c(x, y) = \Vert x - y\Vert^2$, then $\kappa(x, y) =: \phi_\varepsilon(x - y)$ is a Gaussian convolution kernel and the sample measure $\mu = \phi\star\nu^*$ is a convolution, so the solution of \eqref{nu=0} is the MLE of the Gaussian deconvolution problem, as already presented by \citet{rigollet2018entropic}.

As in the Gaussian case, these optimization problems share an optimum, but are not equal in value. Therefore, in our regularized setting, it is not possible to substitute one for the other.
\end{rem}

%% file: Sections/3-dualformulations.tex
\section{Dual formulations}
As noted above, first-order optimization methods to solve directly in $\nu$ the regularized problem require at every step to solve an optimization problem. We explore instead another approach, through a dual formulation of our problem. Such a formulation allows to change the minimisation problem in \eqref{OTproblem} into a maximisation problem.
\begin{prop}[Dual formulation] If $\varepsilon> 0$, then the problem \eqref{problem} is equivalent to the following problem:
 \begin{equation}\label{dualproblem}
     \min_{f\in \mathcal{F}}\max_{a\in L^1(\mu), b\in L^1(\nu)} \mathbb{E} \left[ a(X) + b(Y)f(Y)  -\varepsilon \exp\left( \frac{a(X) + b(Y) - c(X,Y)}{\varepsilon}\right) + (\eta-\varepsilon) f(Y)\log f(Y)\right],
 \end{equation}
 the expectation being over the variables $(X,Y)\sim \mu\otimes\beta$, with $f(y) = \frac{d\nu}{d\beta}(y)$ and $\mathcal{F} = \{\frac{d\nu}{d\beta}\  : \ \nu\in\mathcal{M}\}$.
\end{prop}
If $f$ is constant $\beta$-almost everywhere, with value $1$, then the maximization problem for $a$ and $b$ in \eqref{dualproblem} is the dual of the regularized optimal transport problem \ref{OTproblem}, for which a block coordinate descent corresponds to Sinkhorn algorithm \citep{cuturi2013sinkhorn}.

This dual formulation is a saddle point problem, and it is convex-concave if $\eta \geq\varepsilon$, so the Von Neumann minimax theorem applies: we can swap the minimum and the maximum.

\begin{prop} If  $\eta \geq\varepsilon > 0$ then  the problem \eqref{problem} is equivalent to the following maximization problem:
\begin{equation}\label{finalproblem}
 \max_{a\in L^1(\mu), b\in L^1(\nu)} F(a,b),
\end{equation}
with
\begin{equation}\label{FunctionDef}
     F(a,b)= \Esp{}{ a(X) -\varepsilon e^{\frac{a(X) + b(Y) - c(X,Y)}{\varepsilon}}} -  (\eta-\varepsilon)H^*_\beta\left(-\frac{b}{\eta - \varepsilon}\right),
 \end{equation}
by writing
\begin{equation*}
H^*_\beta\left(\alpha\right) = \max_{f\in \mathcal{F}}\Esp{}{\alpha(Y)f(Y) - f(Y)\log f(Y)},
\end{equation*}
with the variables $(X,Y)\sim \mu\otimes\beta$.
\end{prop}

In its discrete formulation, the problem is written with the following notations: $C_{i,j} := c(x_i,y_j)$ for the cost matrix, $a_i = a(x_i)$ and $b_j = b(y_j)$ for the dual vectors, and $f_j = f(y_j)$ for the remaining primal variable.

The problem \eqref{finalproblem} is hence given by
\begin{equation}\label{discreteproblem}
\max_{(a,b)\in \R^I\times\R^J} F(a,b),
\end{equation}
with
\begin{equation}\label{randomindices}
    F(a,b) = \Esp{}{a_i - \varepsilon  \exp\Big(\frac{a_i + b_j -C_{i,j}}{\varepsilon}\Big)} -(\eta-\varepsilon) H_{\beta,\mathcal{M}}^*\left(-\frac{b}{\eta-\varepsilon}\right).
\end{equation}
The indices $(i,j)$ are here independent random variables such that $x_i \sim \mu$ and $y_j\sim \beta$.
The function $H_{\beta,\mathcal{M}}^*$ is the Legendre transform of the relative entropy to $\beta$ on the set~$\mathcal{F}$:
\begin{align}\label{Legendre of entropy}
    H_{\beta,\mathcal{M}}^*(\alpha)& = \max_{f\in\mathcal{F}}\Esp{}{f_j (\alpha_j - \log{f_j})},
\end{align}
with $j$ a random index such that $y_j\sim \beta$.

 If the maximum is attained on the relative interior of $\mathcal{M}$ at the point $\nu^*(\alpha)$, then we have $\nabla H_{\beta,\mathcal{M}}^*(\alpha) = \nu^*(\alpha).$ Moreover the optimum $\nu^*(-b^*/(\eta-\varepsilon))$ for the dual problem \eqref{dualproblem} is the optimal $\nu\in\mathcal{M}$ for our general problem \eqref{problem}. 

\begin{prop} The function $F$ has the following properties.
\label{PRO:basicprop}
\begin{enumerate}
\item  The set of solutions to the problem \eqref{discreteproblem} is a nonempty affine space spanned by the vector $((1,\dots,1), (-1,\dots,-1))$.
\item Every solution $(a^*,b^*)$ of \eqref{discreteproblem} verifies
\begin{equation}\label{solutionrange}
    \forall\ i,j,\ \vert  a^*_i + b^*_j - C_{i,j}\vert \leq  B ,
\end{equation}
with $B:= \varepsilon m + 2R_C$, where $R_C$ is the range of the matrix $C$ given by $R_C := \max_{i,j}C_{i,j} - \min_{i,j} C_{i,j}$, and $m := \max_j\vert\log f_j\vert$ with $f_j = \nu_j^*/\beta_j$.

\item The function $-F$ is $\lambda$-strongly convex on the slice $\{\sum_i \mu_i a_i=\sum_j \beta_j b_j \}$ with 
\begin{equation*}
    \lambda := \frac{\min_{i,j} \{\mu_i,\beta_j\}{}}{\varepsilon} e^{-(m + 2R_C/\varepsilon)}.
\end{equation*}
\item
For $i$ and $j$ independent random variables as for \eqref{randomindices}, we have the gradients of $F$ are written as simple expectations
\begin{align}\label{stoch1}
    \nabla_aF &= \mathbb{E}\left[(1-D_{i,j})e_i\right],\\
\label{stoch2}
\nabla_bF &= \mathbb{E}\left[(f_j-D_{i,j})e_j\right],
\end{align} 
with $D_{i,j}(a,b) = \exp\Big(\frac{a_i+b_j-C_{i,j}}{\varepsilon}\Big)$.
\end{enumerate}
\end{prop}

%% file: Sections/4-stochoptmethods.tex
\section{Stochastic Optimization Methods}

The formulas \eqref{stoch1} and \eqref{stoch2} suggest that our problem can be solved using a stochastic optimization approach. For random indices $i$ drawn from $\mu$ and $j$ drawn from $\beta$, we obtain the following stochastic gradients
\begin{align*}
G_a& = (1 - D_{i,j})e_i = \left(1 -  \exp\Big(\frac{a_i+b_j-C_{i,j}}{\varepsilon}\Big)\right)e_i \\
G_b& = (f_j - D_{i,j})e_j = \left(\frac{\nu_j^*}{\beta_j} - \exp\Big(\frac{a_i+b_j-C_{i,j}}{\varepsilon}\Big)\right)e_j.
\end{align*}
By Proposition~\ref{PRO:basicprop}, these are unbiased estimates of the gradients of $F$. The algorithm then proceeds with an averaged gradient ascent that uses these stochastic gradients updates at each step. The obtained iterates $(b^t)_{t\geq 1}$ are averaged, producing the sequence $\left(\overline{b^t}\right)_{t\geq 0}$ of iterates defined by
\begin{equation*}
    \overline{b^t} := \frac{1}{t}\sum_{1\leq t'\leq t}b^{t'}.
\end{equation*}
The computation of $G_a$ can be done in $O(1)$, however $G_b$ necessitates the value $\nu_j^*$ in \eqref{Legendre of entropy} to be computed. The complexity of this computation depends on the set $\mathcal{M}$, and we will present here two cases where it can be done with low complexity.

\paragraph{Initialization.} To guarantee that the gradients will not get exponentially big, we choose the initial value of the dual variables so that it verifies
$$\forall i,j,\ a_i+b_j-C_{i,j}\leq -\varepsilon m,$$
with $m$ being defined in \eqref{solutionrange}. We define 
$$\text{ini}(C, \varepsilon, m) := (\min C_{i,j}- \varepsilon m)/2,$$
and we initialize
\begin{equation}\label{initialization}
      a_i = b_j= \overline{b_j}= \text{ini}(C, \varepsilon, m).
\end{equation}
Usually, $m$ is unknown and should be determined by heuristics.
\paragraph{Simple case.}

We analyze the case where $\mathcal{M}$ is the family  of all probability measures supported in the finite set $\{y_j\}_{1\leq j\leq J}\subset \R^d$, with the assumption that $\eta>\varepsilon$. Then, if the max is attained on the interior of the simplex, we have the optimum 
\begin{equation}\label{estimator}
    \nu^*_j = \frac{\beta_j e^{-b_j/(\eta - \varepsilon)}}{\sum_k\beta_ke^{-b_k/(\eta - \varepsilon)}}.
\end{equation}

\begin{algorithm}
\label{SGD1}
\caption{SGD for Wasserstein estimator}
\begin{algorithmic}
\STATE The entries are the learning rates $(\gamma_t)$, the probabilities $\mu = (\mu_i)_i$, $\beta = (\beta_j)_j$, the cost matrix $C_{i,j}$ and the logarithmic gap $m$ between the solution and the prior.
\STATE Initialize $a_i = b_j = \overline{b_j} = \text{ini}(C, \varepsilon, m)$, $S = e^{-\frac{\text{ini}(C, \varepsilon, m)}{\eta - \varepsilon}}.$.
\FOR{$t=1$ to $T$} 
    \STATE Sample $i\in\{1,\dots,I\}$ with probability $\mu_i$.
    \STATE Sample $j\in\{1,\dots,J\}$ with probability $\beta_j$.
    \STATE $D_{i,j} = e^{\frac{a_i + b_j - C_{i,j}}{\varepsilon}}$.
    \STATE $f_j = e^{-b_j/(\eta - \varepsilon)}/S$.
    \STATE $a_i \leftarrow  a_i + \gamma_t (1 - D_{i,j}).$
    \STATE $b_j \leftarrow b_j + \gamma_t (f_j - D_{i,j})$ with the previous as $b_j'$.
    \STATE $\overline{b_j} \leftarrow \left( 1 - \frac{1}{t} \right)\overline{b_j} + \frac{1}{t}b_j$
    \STATE $S \leftarrow S + \beta_j e^{-b_j/(\eta - \varepsilon)} - \beta_j e^{-b_j'/(\eta - \varepsilon)}$
\ENDFOR
\FOR{$j=1$ to $J$}
\STATE $\nu_j = \beta_j e^{-\overline{b_j}/(\eta - \varepsilon)}/\sum_{j'} \beta_{j'} e^{-\overline{b_{j'}}/(\eta - \varepsilon)}$
\ENDFOR
\STATE Return $\nu$.
\end{algorithmic}
\end{algorithm}

The algorithm needs $O(1)$ complexity for each time step. If the values of $C_{i,j}$ are accessible without having the whole matrix stored (such as a simple function of $x_i$ and $y_j$), the storage is only $O(I+J)$ in this algorithm, because we do not need to store any $D_{i,j}$. The complexity at each step of the algorithm is better than with the non regularized form, where $j$ is taken as $\arg \max_j \beta_j e^{-b_j/(\eta - \varepsilon)}$, instead of randomly. This enhancement in complexity mostly comes from the storage of the sum $S^t = \sum_{j} g_j(b_j^t)$ with $$g_j(b_j^t) := \beta_{j} e^{-b_{j}^t/(\eta - \varepsilon)}.$$ Indeed, instead of computing the entire sum at each iterates, which costs $O(J)$ operations, the algorithm simply updates the part of the sum that was modified:
$$S^{t+1} = S^t + g_j(b_j^{t+1}) - g_j(b_j^t).$$
This method assures updates in $O(1)$. In a context focused entirely on optimization, where $\mu$ and $\beta$ are known in advance, we could also pick $i$ and $j$ uniformly, and add $\mu_i$ and $\beta_j$ as factors in the formulas. This would not reduce the complexity.

\paragraph{Mixture models.} 

We also consider a set of measures $(\nu^k)_{1\leq k\leq K} $ supported in supported in the set $\{y_j\}_{1\leq j\leq J}\subset \R^d$, and take $\mathcal{M} = \{\sum_k \theta_k\nu^k : \theta \in \Delta_K\}$ to be their convex hull. We define the matrix $M = (\nu^k(y_j))_{j,k}$. Then $\mathcal{M} = \{M\theta: \theta\in\Delta_K\}$, and Equation (\ref{Legendre of entropy}) becomes
\begin{equation}\label{legendrebarycenter}
    H_{\beta,\mathcal{M}}^*(\alpha) = \max_{\theta\in\Delta_K} (\alpha - \log(M\theta) +\log(\beta) )^T M\theta,
\end{equation}
with the $\log$ being taken component-wise. 
\begin{prop}
The maximization problem \eqref{legendrebarycenter} has a solution
\begin{equation*}
    \theta^* = \frac{M^{\dagger}\exp{\left(P_{\text{Im}(M)}(-b/(\eta - \varepsilon) - 1 - \log(\beta))\right)}}{1^T M^{\dagger} \exp{\left(P_{\text{Im}(M)}(-b/(\eta - \varepsilon) - 1 - \log(\beta))\right)}},
\end{equation*}
which gives the measure
\begin{equation*}
    \nu^* = \frac{\exp{\left(P_{\text{Im}(M)}(-b/(\eta - \varepsilon) - 1 - \log(\beta)\right)}}{1^T \exp{\left(P_{\text{Im}(M)}(-b/(\eta - \varepsilon) - 1 - \log(\beta)\right)}}.
\end{equation*}
\end{prop}

We can replace it in equation \eqref{stoch2} to get the stochastic gradients. However at each new computed step, every coefficient changes, and there is a need to do $J$ computations for each step. The solution computed here is also valid for the case when it is not unique.

We can, however, consider another regularization to the entropy of $\theta$ to improve the algorithm. The problem is the following:
 \begin{equation*}
     \min_{\theta\in\Delta_K} \OT_\varepsilon(\nu,\mu) + \eta \KL(\theta, M^\dagger\beta),
 \end{equation*}
 with $M^\dagger$ being the Moore-Penrose inverse of the matrix $M$.
The other computations are unchanged, apart from Equation~\eqref{Legendre of entropy}, replaced by
\begin{align}\label{legendre2}
    H_{\beta,\mathcal{M}}^*(\alpha) &= \max_{\theta\in\Delta_K} \alpha^TM\theta - (\eta-\varepsilon) \KL(\theta, M^\dagger\beta) \nonumber\\
&= \max_{\theta\in\Delta_K} (M^T\alpha - \log(\theta) + \log( M^\dagger\beta))^T\theta.
\end{align}
\begin{prop}The maximization problem \eqref{legendre2} has a solution
\begin{equation*}
    \theta^* = \frac{\exp{\left(M^T(-b/(\eta - \varepsilon) - 1) + \log( M^\dagger\beta)\right)}}{1^T\exp{\left(M^T(-b/(\eta - \varepsilon) - 1) + \log( M^\dagger\beta)\right)}}.
\end{equation*}
\end{prop}
Both regularizations $\KL(\theta, M^\dagger\beta)$ and $\KL(\nu,\beta)$ are minimal when $\nu = \beta$, and can therefore be used as a suitable proxy. The solution to the regularized problem is similar to the solution to the unregularized one. For this modified problem, the computations are accessible, and they can be done in time $O(K)$, a great improvement if $K\ll J$. 
The algorithm is the following:
\begin{algorithm}
\caption{SGD for Wasserstein projection}
\begin{algorithmic}
\STATE The entries are the learning rates $(\gamma_t)$, the probabilities $\mu = (\mu_i)_i$, $\beta = (\beta_j)_j$, the stochastic matrix $M = (\nu_j^k)_{j,k}$ , the cost matrix $C_{i,j}$ and the logarithmic gap $m$ between the solution and the prior.
\STATE Initialize $a_i$, $b_j$, $\overline{b_j}$, $\alpha = \log\left(M^\dagger\beta\right)$, $\theta_k = 1/K$.
\FOR{$t=1$ to $T$}
    \STATE Sample $i\in\{1,\dots,I\}$ with probability $\mu_i$.
    \STATE Sample $j\in\{1,\dots,J\}$ with probability $\beta_j$.
    \STATE $D_{i,j} = e^{\frac{a_i + b_j - C_{i,j}}{\varepsilon}}$.
    \STATE $f_j =  \sum_{k=1}^K\theta_k\nu_j^k/\beta_j$.
    \STATE $a_i \leftarrow  a_i + \gamma_t (1 - D_{i,j}).$
    \STATE $b_j \leftarrow b_j + \gamma_t (f_j - D_{i,j})$.
    \FOR{$k=1$ to $K$}
        \STATE $\alpha_k\leftarrow \alpha_k -\frac{\gamma_t}{\eta - \varepsilon}\nu_j^k(f_j - D_{i,j}).$
        \STATE $\overline{\alpha_k} \leftarrow \left( 1 - \frac{1}{t} \right)\overline{\alpha_k} + \frac{1}{t}\alpha_k$
    \ENDFOR
    \FOR{$k=1$ to $K$}
        \STATE $\theta_k = e^{\alpha_k}/\sum_{k'} e^{\alpha_{k'}}.$
    \ENDFOR
\ENDFOR
\FOR{$k=1$ to $K$}
    \STATE $\theta_k = e^{\overline{\alpha_k}}/\sum_{k'}e^{\overline{\alpha_{k'}}}.$
\ENDFOR
\FOR{$j=1$ to $J$} 
\STATE $\nu_j =  \sum_{k=1}^K\theta_k\nu_j^k$.
\ENDFOR
\STATE Return $\nu$.
\end{algorithmic}
\end{algorithm}

\paragraph{Wasserstein barycenters.}
Algorithm \ref{SGD1} can be used to compute an approximation of the Wasserstein barycenter of $K$ measures $\mu^1,\dots,\mu^K$. If the cost funtion in the optimal transport problem is of the form $c(x, y) = d(x, y)^p$ with $d$ being a distance and $p\geq 1$, then the transport cost $\OT(\cdot,\cdot)$ defines the $p$-Wasserstein distance. In these conditions, the Wasserstein barycenter of the measures  $\mu_1,\dots,\mu_K$ with weights  $\theta^1,\dots,\theta^K$ is the solution of the problem
\begin{equation}\label{WassBar}
    \min_{\nu} \sum_{k=1}^K \theta_k \OT(\mu^k, \nu).
\end{equation} 
This optimization and the barycenter that it defines was introduced by \citet{agueh2011barycenters}, these objects and their regularized versions have attracted a lot of attention, for their statistical and algorithmic aspects \citep{zemel2019frechet,cuturi2014fast,claici2018stochastic,luise2019sinkhorn}.

As an analogy with our original problem \eqref{problem}, we consider an entropic regularization of the Wasserstein barycenter problem \eqref{WassBar}:
\begin{equation*}
    \min_{\nu\in\mathcal{M}} \sum_{k=1}^K \theta_k \OT_\varepsilon(\mu^k, \nu) + \eta \KL(\nu, \beta).
\end{equation*}
Our approach can be translated to this setting, as well as the theoretical results found for \eqref{problem}. We have the equivalent dual formulation
\begin{equation*}
 \max_{a\in L^1(\mu), b\in L^1(\nu)} \Tilde{F}(a^1,\dots, a^K,b),
\end{equation*}
with
\begin{equation*}
 \Tilde{F}(a^1,\dots, a^K,b) := \sum_{k=1}^K\theta_k F_k(a^k, b).
\end{equation*}
Here $F_k$ is defined like the function $F$ in \eqref{FunctionDef} by replacing $\mu$ by $\mu^k$.
The only difference in the algorithm is that there should be $K$ dual variables $a^1,\dots, a^K$ that play the role of the variable $a$ for each measure $\mu^k$ while one variable $b$ is used to obtain the target measure. The complexity of the algorithm is $O(K)$ for each stochastic gradient step, which gains a factor $\log K$ compared to the state-of-the-art stochastic Wasserstein barycenter \citep{staib2017parallel}, that solves the minimisation problem
\begin{equation*}
    \min_{\nu\in\mathcal{M}} \sum_{k=1}^K \theta_k \OT_\varepsilon(\mu^k, \nu).
\end{equation*}
The complexity of a gradient step could be further reduced to $O(1)$ at the cost of more randomization, by sampling $k$ randomly at each step with probability $\theta_k$, and updating $a_k$ and $b$ as in algorithm \ref{SGD1} with $\mu_k$ playing the role of $\mu$. If $\eta\approx\varepsilon$, the approximation error of this estimated Wasserstein Barycenter is of the same order as by \citet{staib2017parallel}.

%% file: Sections/5-results.tex
\section{Results}

\subsection{Convergence bounds}

The following convergence bounds are valid for both algorithms presented in the previous section. They come from general convergence bounds averaged stochastic gradient descent with decreasing stepsize \citep{shamir2012stochastic}. For $\nu^*\in\mathcal{M}$ be the optimal Wasserstein estimator, let $\nu^t\in\mathcal{M}$ be the estimator obtained by stopping the algorithm at step $t$. We consider the Kullback-Leibler divergence to express how close the estimated measure $\nu^t$ is to $\nu^*$. As $\nu^t$ is obtained with the dual variable $b^t$, the estimation error of $b^t$ can translate to an entropic error in the following two bounds. The first result uses the stepsize for SGD associated to strongly convex functions and the second one uses the stepsize for SGD associated to convex functions. Both results are presented here: even though the theoretical bound of the second one is asymptotically worse, its stepsize can yield better performance in practice.
\begin{thm}\label{thm1} With stepsize $\gamma_t = \frac{1}{\lambda t}$, the estimator verifies the following bound:
$$\mathbb{E}\left[\KL(\nu^*,\nu^t)\right]\leq 34\frac{e^{2m} }{(\eta - \varepsilon)\lambda^2} \frac{1+\log t}{t}.$$
\end{thm}
\begin{thm}\label{thm2}
With stepsize $\gamma_t = \frac{c_0\varepsilon}{\sqrt{t}}$, $c_0\leq Be^{-m}/\varepsilon$, the estimator verifies the following bound:
$$\mathbb{E}\left[\KL(\nu^*,\nu^t)\right]\leq 2\frac{B^2 e^m }{c_0\varepsilon(\eta - \varepsilon)\lambda} \frac{2+\log t}{\sqrt{t}}.$$
\end{thm}
In order to prove both theorems, we present two lemmas whose proofs are provided in the appendix.
\begin{lem}\label{boundgrad}
Let $a^t, b^t$ be the iterations of the stochastic gradient descent, seen as random variables. If the initialization is done as in \eqref{initialization}, then the second order moments of the stochastic gradients are bounded:
    	\begin{equation*}
        		\mathbb{E}\left[\Vert\nabla_a F_{i,j}(a^t,b^t)\Vert^2 + \Vert\nabla_b F_{i,j}(a^t,b^t)\Vert^2\right] \leq 2e^{2m}.
  	\end{equation*}
\end{lem}
\begin{lem}\label{KLtoTarget}
The convergence of the primal variable $\nu(b)$ is linked to the convergence of the objective by the following bound:
\begin{equation*}
    \text{KL}(\nu(b^*),\nu(b))\leq \frac{F(a^*,b^*) - F(a, b)}{(\eta-\varepsilon)\lambda}.
\end{equation*}
\end{lem}
\begin{proof}[Proof of Theorem \ref{thm1}]
 The result from \citet{shamir2012stochastic} on strongly convex functions gives the bound
$$\mathbb{E}\left[F(a^*,b^*) - F(a^t, b^t)\right]\leq 17\frac{G^2 }{\lambda} \frac{1+\log t}{t},$$
with $G^2$ being a bound on the second order moments of the stochastic gradients. The lemma \ref{boundgrad} provides $G^2 = 2e^{2m}$.
We conclude with lemma \ref{KLtoTarget}.
\end{proof}
\begin{proof}[Proof of theorem \ref{thm2}]
With stepsize $\gamma_t = \frac{B}{G\sqrt{t}}$, the result from \citet{shamir2012stochastic} on convex functions gives the bound
$$\mathbb{E}\left[F(a^*,b^*) - F(a^t, b^t)\right]\leq 2 (BG) \frac{2+\log t}{\sqrt{t}},$$
with $G^2$ being a bound on the second order moments of the stochastic gradients. The lemma \ref{boundgrad} provides $G\geq \sqrt{2}e^{m}$, here we choose $G = \frac{B}{c_0\varepsilon}$ where we assume $c_0\leq Be^{-m}/\varepsilon$.
We conclude with Lemma \ref{KLtoTarget}.
\end{proof}
\begin{rem} The term in $\log t$ can be removed by using adaptive averaging schemes: by averaging only the past $\alpha t$ iterates, the the term $1+\log t$ can be replaced by $\frac{1 - \log(1-\alpha)}{\alpha}$.
\end{rem}
\begin{rem}
The strong convexity coefficient 
$$\lambda = \frac{\min_{i,j} \{\mu_i,\beta_j\}{}}{\varepsilon} e^{-B/\varepsilon}$$
is negligible when $\varepsilon\ll B$, thus the stepsize of the first theorem is large: it can lead to growth of the dual variables grow out of their normal range and produces an exponential overflow in experiments. One solution is to cap the dual variables to the range provided by \eqref{solutionrange}, but the algorithm would then not provide any useful solution until a high number of steps is performed, i.e. $t\gtrapprox 1/B\lambda$. Instead, we recommend using the stepsize $\gamma_t = \min\{1/\lambda t, c_0\varepsilon/\sqrt{t}\}$ that provides a quick convergence at the earlier steps, then gives a better asymptotic convergence rate.

\end{rem}

\subsection{Simulations}

We demonstrate the performance of the algorithm on simulated experiments.

\paragraph{Regularization term.} In order to exhibit clearly the impact of regularization parameters, We analyze a simple case, where $\mathcal{X} = \mathcal{Y}$, and $C_{i,j} = \vert i - j\vert$. In this case the solution is given by $\nu^* = \mu$ for $\varepsilon = \eta = 0$, with a diagonal transportation matrix. The introduction of the positive regularization in $\eta$
noticeably spreads the transportation matrix, and provides a solution that is closer to the uniform law on $\mathcal{Y}$. We use the learning rate provided by Theorem~\ref{thm2}.
\begin{figure}[!ht]
\centering
\includegraphics[width=5cm]{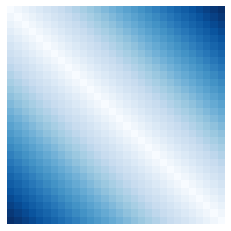}
\includegraphics[width=5cm]{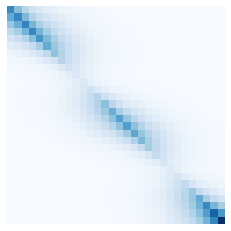}
\includegraphics[width=5cm]{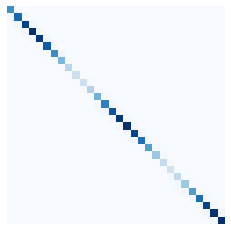}
\includegraphics[width=7.5cm]{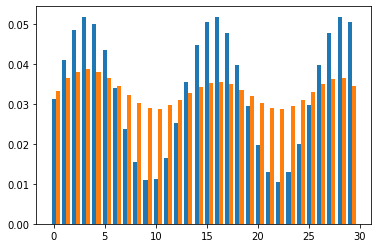}
\includegraphics[width=7.5cm]{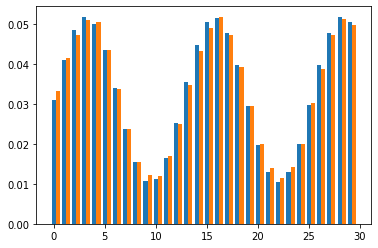}
\caption{\label{tab1}Effect of the regularization. \textit{Upper plots, from left to right:} cost matrix used, transportation matrix for $\varepsilon=\eta - \varepsilon=0.1$ after $10^5$ iterations, and for $\varepsilon=\eta - \varepsilon=0.01$. \textit{Lower plots, from left to right:} Target measure $\mu$ in blue, estimator in orange, $\varepsilon=\eta - \varepsilon=0.1$, then for $\varepsilon=\eta - \varepsilon=0.01$. }
\end{figure}

The regularization term $\eta$ should be greater than $\varepsilon$, and brings the estimated measure closer to the uniform measure. We choose to take $\eta = 2\varepsilon$ to conserve a similar degree of regularization as in the case $\eta=0$, while guaranteeing that the exponentials in \eqref{estimator} do not overflow.

\paragraph{Sensibility to dimension.} We consider the relationship between the convergence rate and the dimensions $(I,J)$ of the problem. The theoretical results \ref{thm1} and \ref{thm2} depend on $( \min_i\mu_i) + (\min_j\beta_j)$, which scales with $1/\min(I,J)$ if $\mu$ and $\beta$ are uniform on their support. We generate $\mathcal{X}$ and $\mathcal{Y}$ as two samples of $I$ and $J$ independent Gaussian vectors, $\mu$ is the uniform measure on $\mathcal{X}$, and $C_{i,j}$ is the distance matrix between $X_i$ and $Y_j$. We compute the gradient norm of the objective function $F$ at the averaged iterates $\overline{a_t},$ $\overline{b_t}$.
\begin{figure}[!h]
\centering
\includegraphics[width=10cm]{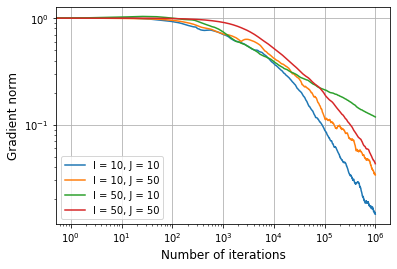}
\caption{\label{tab2}Convergence of the gradient norm for different dimensions.}
\end{figure}

The gradient norm here converges at rate $O(T^{-\delta})$, with $\delta \geq 1/2$ as would be predicted from the theorem \ref{thm1}, except in the case $I>J$ where $1/4\leq \delta < 1/2$, which matches better with the bound in Theorem \ref{thm2}. An increase of the sample size $I$ for the input measure seems to decrease performance while an increase of the support size $J$ of the target increases performance. It means that a finer grid of points in $\mathcal{Y}$ will provide a faster convergence to the optimal estimator.

\paragraph{Choice of the learning rate.} As noted above, a choice of learning rate that is large compared to $\varepsilon$ can lead to a divergence of the dual variables. This is due to the exponential dependency of the gradients in $a$ and $b$. Experiments suggest the learning rate  $$\gamma_t = \min\left\{ \frac{1}{\lambda t}, \frac{c_0\varepsilon}{\sqrt{t}}\right\}.$$ The following graphs show the convergence to the target with different choices of $c_0$. Here $\varepsilon = 0.001$, $\eta = 0.002$, with the same problem is the same as in the experiments on the regularization term. 
\begin{figure}[!h]
\centering
\includegraphics[width=10cm]{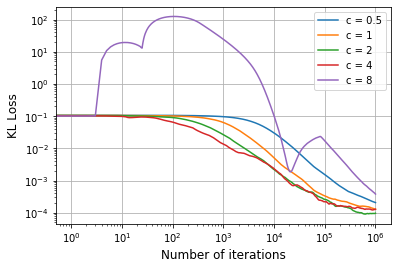}
\caption{\label{tab3}Comparison of the learning rates.}
\end{figure}

A regression on the curves shows that the empirical convergence rate is of order $O\left(T^{-\delta}\right)$ with $\delta >1$, which matches with theorem \ref{thm1}.
We remark that the greater $c_0$ is, the better the algorithm converges, until it becomes unstable and does not converge anymore for $c_0 > 5$. This instability was observed consistently for a large range of values of $\varepsilon$ and $\eta$. The choice $c_0=2$ appears to be reasonable for both stability and convergence.

%% file: Sections/6-conclusion.tex
\section{Conclusion}
We consider the problem of minimizing a doubly regularized optimal transport cost over a set of finitely supported measures with fixed support. Using an entropic regularization on the target measure, we derive a stochastic gradient descent on the dual formulation with sublinear (even constant in the simplest case) complexity at each step of the optimization. The algorithm is thus highly paralellizable, and can be used to compute a regularized solution to the Wasserstein barycenter problem. We also provide convergence bounds for the estimator that this algorithm yields after $t$ steps, and demonstrate its performs on randomly generated data.

\paragraph{Aknowledgements.} This work was funded in part by the French government under management of Agence Nationale de la Recherche as part of the ``Investissements d'avenir'' program, reference ANR-19-P3IA-0001 (PRAIRIE 3IA Institute). We also acknowledge  support from the European Research Council (grant SEQUOIA 724063).

%% file: Sections/A-appendix.tex
\appendix
\newpage

\section{Proofs of technical results}
\begin{proof}[Proof of proposition 2.1]
This proof follows the reasoning in \cite{rigollet2018entropic}. Let $\mu = \frac{1}{I}\sum_i \delta_{X_i}$ be the empirical measure of the sample $(X_i)$. We first remark that the log-likelihood of $X_i$ defined by
$$\ell_\nu(X_i):= \log \int \kappa(X_i, y) d\nu(y) $$
verifies 
$$\ell_\nu(X_i) = \log \mathbb{E}_{Y\sim \nu}\left[\kappa(X_i, Y)\right].$$
With the Legendre transform of the relative entropy, we obtain
$$\ell_\nu(X_i)  = \sup_{\gamma_i} \mathbb{E}_{Y\sim \gamma_i}\left[\log\kappa(X_i, Y)\right] - \text{KL}(\gamma_i, \nu) $$
with the minimum being over every probability measures $\gamma_i$ on $\mathcal{Y}$. The MLE maximizes
$$\frac{1}{I}\sum_i \ell_\nu(X_i) = \mathbb{E}_{X\sim\mu}\left[\ell_\nu(X)\right] $$
over $\nu\in\mathcal{M}$, it can be written as
$$\max_{\pi\in\Pi(\mu,\nu)}\mathbb{E}_{(X,Y)\sim \pi}\left[\log\kappa(X,Y)\right] - \mathbb{E}_{X\sim \mu}\left[KL(\pi(X,\cdot), \nu)\right], $$
with $\pi(X,\cdot)$ being the conditional probability of $\pi$, defined by $\pi(X_i,\cdot) := \gamma_i$. We have
\begin{align*}
    \mathbb{E}_{X\sim \mu}\left[KL(\pi(X,\cdot), \nu)\right] = &\frac{1}{I}\sum_i \mathbb{E}_{Y\sim \nu}\left[\log \frac{d\pi(X_i, \cdot)}{d\nu}(Y)\right],\\
= &\frac{1}{I}\sum_i \mathbb{E}_{Y\sim \nu}\left[\log \frac{d\pi}{d\mu\otimes\nu}(X_i,Y)\right]- \log I,\\
    = &KL(\pi, \mu\otimes\nu) -\log I.
\end{align*}
Thus the MLE minimizes
$$\min_{\pi\in\Pi(\mu,\nu)}\mathbb{E}\left[c(X,Y)\right] + \varepsilon \text{KL}(\pi, \mu\otimes\nu), $$
which is the regularized optimal transport cost between $\mu$ and $\nu$.
\end{proof}



\begin{proof}[Proof of Proposition 3.3]
\begin{enumerate}
    \item The function $H_{\beta,\mathcal{M}}^*$ is a Legendre transform, so it is convex, and thus $-F$ is convex as a sum of convex functions. Moreover $F$ is bounded from above:
    \begin{align*}
    F(a,b) \leq & C_1\mathbb{E}[a_i + b_j] - C_2\mathbb{E}\left[e^{\frac{a_i + b_j}{\varepsilon}}\right],\\
    \leq & C_3,
\end{align*}
where $C_3$ does not depend on $a$ or $b$. Thus the set of solutions is nonempty.
    F is invariant by the translation $(a, b)\mapsto (a_1+c, \dots, a_I+c, b_1-c, \dots, b_J - c)$, so each solution generates an affine set of solutions spanned by the vector $((1,\dots,1), (-1,\dots,-1))$. We can conclude using the strong convexity on the slice $\{\sum_i \mu_i a_i=\sum_j \beta_j b_j \}$, which implies that there exists only one solution on this slice.
    \item The solution $(a^*, b^*)$  solves the following system
\begin{equation*}
    \begin{cases}
    \nabla_a F(a^*,b^*) = 0,\\
    \nabla_b F(a^*,b^*) = 0.
    \end{cases}
\end{equation*}
With notations $A_i = e^{a^*_i/\varepsilon}$, $B_j = e^{b^*_j/\varepsilon}$, $\Gamma_{i,j} = e^{-C_{i,j}/\varepsilon}$, the two equations can be written as
\begin{equation}\label{system1}
    \begin{cases}
    \forall\  1\leq i\leq I,  &1 - A_i\sum_j\beta_j B_j\Gamma_{i,j} = 0,\\
    \forall\  1\leq j\leq J, &f_j - B_j\sum_i\mu_i A_i\Gamma_{i,j} = 0.
    \end{cases}
\end{equation}
Thus
\begin{equation}\label{system2}
    \begin{cases}
    \forall\  1\leq i\leq I,  &A_i = \frac{1}{\sum_j\beta_j B_j\Gamma_{i,j}},\\
    \forall\  1\leq j\leq J, &B_j = \frac{f_j}{\sum_i\mu_i A_i\Gamma_{i,j}}.
    \end{cases}
\end{equation}
We also remark that by multiplying the second term of \eqref{system1} by $\beta_j$ and summing over $j$ we get
\begin{equation}\label{thesum}
    \sum_{i,j} \mu_i A_i\beta_j B_j \Gamma_{i,j} = 1.
\end{equation}
By multiplying the equations in \eqref{system2} we have for all $i, j$:
$$ A_iB_j\Gamma_{i,j} = \frac{f_j\Gamma_{i,j}}{\sum_{k,l}\mu_k A_k\Gamma_{k,j}\beta_l B_l\Gamma_{i,l}} $$
thus using \eqref{thesum}:
$$ f_j\min_{k,l}\frac{\Gamma_{i,j}\Gamma_{k,l}}{\Gamma_{k,j}\Gamma_{i,l}} \leq A_iB_j\Gamma_{i,j} \leq f_j\max_{k,l}\frac{\Gamma_{i,j}\Gamma_{k,l}}{\Gamma_{k,j}\Gamma_{i,l}}$$
finally
$$e^{-m - 2 R_C/\varepsilon} \leq A_iB_j\Gamma_{i,j}\leq e^{m + 2 R_C/\varepsilon}.$$

\item We now prove that $-F$ is strongly convex. We compute
\begin{equation*}
    -\nabla_a^2F = \mathbb{E}\left[\frac{1}{\varepsilon} D_{i,j}E_{i,i}\right],\end{equation*}
    \begin{equation*}
    -\nabla_b^2F = -\nabla_b \nu^* +\mathbb{E}\left[\frac{1}{\varepsilon} D_{i,j}E_{j,j}\right],\end{equation*}
\begin{equation*}
    -\nabla_a\nabla_bF = \mathbb{E}\left[\frac{1}{\varepsilon} D_{i,j}E_{i,j}\right].\end{equation*}
We remark that
\begin{equation*}
    \nu^* = \text{softmax}(-b_j/\eta + \log \beta_j),
\end{equation*}
so
\begin{equation*}
    -\nabla_b \nu^* = \frac{1}{\eta} S
\end{equation*}
with 
\begin{equation*}
    S := (\nabla \text{softmax})(-b_j/\eta + \log \beta_j),
\end{equation*}
\begin{equation*}
    S = (\nu_i(\delta_{i,j} - \nu_j))_{i,j}.
\end{equation*}
We remark that $S\succcurlyeq 0$ since 
\begin{align*}
    u^TSu &= \sum_i \nu_i u_i^2 - \left( \sum_i \nu_i u_i\right)^2\\
    &= \mathbb{E}_\nu[U^2] - \left(\mathbb{E}_\nu[U]\right)^2\geq 0
\end{align*}
by Jensen, with $U = u_j$ with probability $\nu_j$. It implies $-\nabla_b \nu_j^* \succcurlyeq 0$. So 
\begin{equation*}
    -\nabla_{a,b}^2F \succcurlyeq \frac{1}{\varepsilon} M,
\end{equation*}
with
$$ M := \mathbb{E}\left[D_{i,j}\left(\begin{matrix}
    E_{i,i}& E_{i,j} \\
E_{j,i} & E_{j,j}
\end{matrix}\right)\right]. $$
As we want to prove strong convexity on the slice $\sum_i \mu_i a_i = \sum_j \beta_j b_j$, we compute
\begin{equation*}
    (a,b)^T M (a,b) = \mathbb{E}\left[ D_{i,j} (a_i + b_j)^2\right]\geq e^{-B/\varepsilon}\mathbb{E}\left[ (a_i + b_j)^2\right].
\end{equation*} We add that 
\begin{equation*}
    \mathbb{E}\left[ (a_i + b_j)^2\right] = \sum_i \mu_i a_i^2 + \sum_j \beta_j b_j^2 + 2(\sum_i \mu_i a_i)(\sum_j \beta_j b_j)
\end{equation*} 
thus 
\begin{equation*}
    \mathbb{E}\left[ (a_i + b_j)^2\right] = \sum_i (\mu_i +\mu_i^2) a_i^2 + \sum_j (\beta_j + \beta_j^2) b_j^2
\end{equation*} 
since we are on the slice. So $M\succcurlyeq \lambda \text{Id}$ and finally $-F$ is $\lambda-$strongly convex with 
\begin{equation*}
    \lambda = \frac{\min_{i,j}\{\mu_i, \beta_j\}}{\varepsilon} e^{-B/\varepsilon}.
\end{equation*}
\item We compute the gradients of $F$:
\begin{align}
    \frac{\partial F}{\partial a_i}(a,b) &= \mu_i - \mu_i\sum_{j=1}^{J}\beta_j D_{i,j}(a,b),\\
    \label{bgradient}
    \frac{\partial F}{\partial b_j}(a,b) &= \nu_j^*(-b/\eta) - \beta_j \sum_{i=1}^I \mu_i D_{i,j}(a,b),
\end{align}
with $D_{i,j}(a,b) = e^{\frac{a_i+b_j-C_{i,j}}{\varepsilon}}$.
If we take $i$ and $j$ to be independent random variables following the laws $(\mu_i)$ and $(\beta_j)$ respectively, we have the desired expression for the gradients.
\end{enumerate}
\end{proof}

\begin{proof}[Proof of Lemma 1] With the initial conditions, we guarantee that $0\leq G_a^0 \leq 1$ and $0\leq G_b^0\leq f_j\leq e^m$. At each timestep $t$, we have
$$ \Vert\nabla F_{i,j}^t\Vert^2 \leq \max\{2e^{2m}, 2(D_{i,j}^t)^2\}, $$
with $i,j$ being two independent random variables following the laws $\mu$ and $\beta$ respectively.
If $D_{i,j}^t\geq e^m$, then $G_a + G_b\leq 0$ and
$$D_{i,j}^{t+1} = D_{i,j}^t e^{\frac{G_a + G_b}{\varepsilon}}\leq D_{i,j}^t. $$
Moreover if $D_{i,j}^t\leq e^m$ then $\Vert\nabla F_{i,j}^t\Vert^2\leq 1+e^{2m}$
thus $\mathbb{E}\left[\max\{2e^{2m}, (D_{i,j}^t)^2\}\right]$ is a decreasing function of $t$. Thus we have the bound    	
\begin{equation*}
        		\mathbb{E}\left[\Vert\nabla_a F_{i,j}(a^t,b^t)\Vert^2 + \Vert\nabla_b F_{i,j}(a^t,b^t)\Vert^2\right] \leq 2e^{2m}.
  	\end{equation*}
\end{proof} 


\begin{proof}[Proof of Lemma 2]
We first assume that $(a,b)$ and $(a^*,b^*)$ are on the slice $\{\sum_i \mu_i a_i=\sum_j \beta_j b_j \}$.
By strong convexity of $-F$ on this slice we have 
\begin{equation}\label{ineqlem}
    \vert b - b^*\vert^2 \leq \frac{2(F(a^*,b^*) - F(a,b))}{\lambda}.
\end{equation}
We remark that the function $g:b\mapsto KL(\nu(b^*), \nu(b))$ verifies
\begin{align*}
    \partial_i g(b)
    &= -\sum_j \nu_j(b^*)\partial_i\log\nu_j(b),\\
    &= -\sum_j \nu_j(b^*)\nu_j(b)^{-1}\partial_i\nu_j(b),\\
    &= \frac{1}{\eta}\sum_j \nu_j(b^*)\nu_j(b)^{-1}\nu_i(\delta_{ij} -\nu_j(b)) ,\\
    &= \frac{\nu_i(b^*)-\nu_i(b)}{\eta-\varepsilon} ,
\end{align*}
thus
\begin{align*}
    \partial_i\partial_j g(b)
    &= -\frac{\partial_j\nu_i(b)}{\eta-\varepsilon} ,\\
    &= -\frac{\nu_j(b)(\delta_{ij} -\nu_i(b))}{\eta-\varepsilon},
\end{align*}
so the Hessian matrix $\nabla^2g(b)$ of $g$ is a sum of a diagonal matrix with the negative values $-\nu_j(b)/ (\eta-\varepsilon)$ and the one-rank matrix $(\nu_j(b)\nu_i(b)/ (\eta-\varepsilon))_{i,j}$. Hence the eigenvalues of $\nabla^2g(b)$ are contained in $[-1/(\eta-\varepsilon),1/(\eta-\varepsilon)]$, thus Taylor's inequality gives
\begin{align*}
    g(b)&\leq g(b^*) + \vert b - b^*\vert \Vert \nabla g(b^*)\Vert + \frac{\vert b - b^*\vert^2}{2(\eta-\varepsilon)},\\
    &\leq \frac{\vert b - b^*\vert^2}{2(\eta-\varepsilon)},
\end{align*}
because $g(b^*)=0$ and $\nabla g(b^*)=0$.
We complete the proof with \eqref{ineqlem}. For the case where the vector $(a,b)$ or $(a^*,b^*)$ is not on the slice $\{\sum_i \mu_i a_i=\sum_j \beta_j b_j \}$, we note that adding a constant vector $c = (c_1,\dots,c_1) $ to $b$ does not change the value of $\nu(b)$, and that $F$ is invariant by translation in the direction $(-c, +c)$. With $c_1 = \left(\sum_i \mu_i a_i - \sum_j \beta_j b_j\right)/2$, the vectors $(a', b') = (a + c, b - c)$ are on the slice and verify $\nu(b')=\nu(b)$ and $F(a', b')=F(a, b)$. Hence the result for $(a', b')$ implies the result for $(a, b)$.
\end{proof}

%% file: main.bbl
\begin{thebibliography}{41}
\providecommand{\natexlab}[1]{#1}
\providecommand{\url}[1]{\texttt{#1}}
\expandafter\ifx\csname urlstyle\endcsname\relax
  \providecommand{\doi}[1]{doi: #1}\else
  \providecommand{\doi}{doi: \begingroup \urlstyle{rm}\Url}\fi

\bibitem[Agueh and Carlier(2011)]{agueh2011barycenters}
Martial Agueh and Guillaume Carlier.
\newblock Barycenters in the wasserstein space.
\newblock \emph{SIAM Journal on Mathematical Analysis}, 43\penalty0
  (2):\penalty0 904--924, 2011.

\bibitem[Alaux et~al.(2019)Alaux, Grave, Cuturi, and
  Joulin]{alaux2018unsupervised}
Jean Alaux, Edouard Grave, Marco Cuturi, and Armand Joulin.
\newblock Unsupervised hyper-alignment for multilingual word embeddings.
\newblock 2019.

\bibitem[Altschuler et~al.(2017)Altschuler, Niles-Weed, and
  Rigollet]{altschuler2017near}
Jason Altschuler, Jonathan Niles-Weed, and Philippe Rigollet.
\newblock Near-linear time approximation algorithms for optimal transport via
  sinkhorn iteration.
\newblock In \emph{Advances in Neural Information Processing Systems}, pages
  1964--1974, 2017.

\bibitem[Alvarez-Melis et~al.(2018)Alvarez-Melis, Jaakkola, and
  Jegelka]{alvarez2018structured}
David Alvarez-Melis, Tommi Jaakkola, and Stefanie Jegelka.
\newblock Structured optimal transport.
\newblock In \emph{International Conference on Artificial Intelligence and
  Statistics}, pages 1771--1780, 2018.

\bibitem[Arbel et~al.(2019)Arbel, Korba, Salim, and Gretton]{arbel2019maximum}
Michael Arbel, Anna Korba, Adil Salim, and Arthur Gretton.
\newblock Maximum mean discrepancy gradient flow.
\newblock In \emph{Advances in Neural Information Processing Systems}, pages
  6481--6491, 2019.

\bibitem[Arjovsky et~al.(2017)Arjovsky, Chintala, and
  Bottou]{arjovsky2017wasserstein}
Martin Arjovsky, Soumith Chintala, and L{\'e}on Bottou.
\newblock Wasserstein generative adversarial networks.
\newblock In \emph{International Conference on Machine Learning}, pages
  214--223, 2017.

\bibitem[Bassetti et~al.(2006)Bassetti, Bodini, and
  Regazzini]{bassetti2006minimum}
Federico Bassetti, Antonella Bodini, and Eugenio Regazzini.
\newblock On minimum kantorovich distance estimators.
\newblock \emph{Statistics \& probability letters}, 76\penalty0 (12):\penalty0
  1298--1302, 2006.

\bibitem[Berthet and Kanade(2019)]{berthet2019statistical}
Quentin Berthet and Varun Kanade.
\newblock Statistical windows in testing for the initial distribution of a
  reversible markov chain.
\newblock pages 246--255, 2019.

\bibitem[Boursier and Perchet(2019)]{boursier2019private}
Etienne Boursier and Vianney Perchet.
\newblock Private learning and regularized optimal transport.
\newblock \emph{arXiv preprint arXiv:1905.11148}, 2019.

\bibitem[Claici et~al.(2018)Claici, Chien, and Solomon]{claici2018stochastic}
Sebastian Claici, Edward Chien, and Justin Solomon.
\newblock Stochastic wasserstein barycenters.
\newblock In \emph{International Conference on Machine Learning}, pages
  999--1008, 2018.

\bibitem[Clarkson et~al.(2012)Clarkson, Hazan, and
  Woodruff]{clarkson2012sublinear}
Kenneth~L. Clarkson, Elad Hazan, and David~P Woodruff.
\newblock Sublinear optimization for machine learning.
\newblock \emph{Journal of the ACM (JACM)}, 59\penalty0 (5):\penalty0 1--49,
  2012.

\bibitem[Cuturi(2013)]{cuturi2013sinkhorn}
Marco Cuturi.
\newblock Sinkhorn distances: Lightspeed computation of optimal transport.
\newblock In \emph{Advances in neural information processing systems}, pages
  2292--2300, 2013.

\bibitem[Cuturi and Doucet(2014)]{cuturi2014fast}
Marco Cuturi and Arnaud Doucet.
\newblock Fast computation of wasserstein barycenters.
\newblock In \emph{International Conference on Machine Learning}, pages
  685--693, 2014.

\bibitem[del Barrio et~al.(2019)del Barrio, Inouzhe, Loubes, Matr{\'a}n, and
  Mayo-{\'I}scar]{del2019optimalflow}
Eustasio del Barrio, Hristo Inouzhe, Jean-Michel Loubes, Carlos Matr{\'a}n, and
  Agust{\'\i}n Mayo-{\'I}scar.
\newblock optimalflow: Optimal-transport approach to flow cytometry gating and
  population matching.
\newblock \emph{arXiv preprint arXiv:1907.08006}, 2019.

\bibitem[Dvurechensky et~al.(2018)Dvurechensky, Gasnikov, and
  Kroshnin]{dvurechensky2018computational}
Pavel Dvurechensky, Alexander Gasnikov, and Alexey Kroshnin.
\newblock Computational optimal transport: Complexity by accelerated gradient
  descent is better than by sinkhorn’s algorithm.
\newblock In \emph{International Conference on Machine Learning}, pages
  1367--1376, 2018.

\bibitem[Feydy et~al.(2017)Feydy, Charlier, Vialard, and
  Peyr{\'e}]{feydy2017optimal}
Jean Feydy, Benjamin Charlier, Fran{\c{c}}ois-Xavier Vialard, and Gabriel
  Peyr{\'e}.
\newblock Optimal transport for diffeomorphic registration.
\newblock In \emph{International Conference on Medical Image Computing and
  Computer-Assisted Intervention}, pages 291--299. Springer, 2017.

\bibitem[Feydy et~al.(2019)Feydy, S{\'e}journ{\'e}, Vialard, Amari, Trouve, and
  Peyr{\'e}]{feydy2019interpolating}
Jean Feydy, Thibault S{\'e}journ{\'e}, Fran{\c{c}}ois-Xavier Vialard, Shun-ichi
  Amari, Alain Trouve, and Gabriel Peyr{\'e}.
\newblock Interpolating between optimal transport and mmd using sinkhorn
  divergences.
\newblock In \emph{The 22nd International Conference on Artificial Intelligence
  and Statistics}, pages 2681--2690, 2019.

\bibitem[Flamary et~al.(2018)Flamary, Cuturi, Courty, and
  Rakotomamonjy]{flamary2018wasserstein}
R{\'e}mi Flamary, Marco Cuturi, Nicolas Courty, and Alain Rakotomamonjy.
\newblock Wasserstein discriminant analysis.
\newblock \emph{Machine Learning}, 107\penalty0 (12):\penalty0 1923--1945,
  2018.

\bibitem[Genevay et~al.(2016)Genevay, Cuturi, Peyr{\'e}, and
  Bach]{genevay2016stochastic}
Aude Genevay, Marco Cuturi, Gabriel Peyr{\'e}, and Francis Bach.
\newblock Stochastic optimization for large-scale optimal transport.
\newblock In \emph{Advances in neural information processing systems}, pages
  3440--3448, 2016.

\bibitem[Genevay et~al.(2017)Genevay, Peyr{\'e}, and
  Cuturi]{genevay2017learning}
Aude Genevay, Gabriel Peyr{\'e}, and Marco Cuturi.
\newblock Learning generative models with sinkhorn divergences.
\newblock \emph{arXiv preprint arXiv:1706.00292}, 2017.

\bibitem[Genevay et~al.(2019)Genevay, Chizat, Bach, Cuturi, and
  Peyr{\'e}]{genevay2019sample}
Aude Genevay, L{\'e}na{\"\i}c Chizat, Francis Bach, Marco Cuturi, and Gabriel
  Peyr{\'e}.
\newblock Sample complexity of sinkhorn divergences.
\newblock In \emph{The 22nd International Conference on Artificial Intelligence
  and Statistics}, pages 1574--1583, 2019.

\bibitem[Gordaliza et~al.(2019)Gordaliza, Del~Barrio, Fabrice, and
  Loubes]{gordaliza2019obtaining}
Paula Gordaliza, Eustasio Del~Barrio, Gamboa Fabrice, and Jean-Michel Loubes.
\newblock Obtaining fairness using optimal transport theory.
\newblock In \emph{International Conference on Machine Learning}, pages
  2357--2365, 2019.

\bibitem[Grave et~al.(2019)Grave, Joulin, and Berthet]{grave2019unsupervised}
Edouard Grave, Armand Joulin, and Quentin Berthet.
\newblock Unsupervised alignment of embeddings with wasserstein procrustes.
\newblock In \emph{The 22nd International Conference on Artificial Intelligence
  and Statistics}, pages 1880--1890, 2019.

\bibitem[Gretton et~al.(2012)Gretton, Borgwardt, Rasch, Sch{\"o}lkopf, and
  Smola]{gretton2012kernel}
Arthur Gretton, Karsten~M Borgwardt, Malte~J Rasch, Bernhard Sch{\"o}lkopf, and
  Alexander Smola.
\newblock A kernel two-sample test.
\newblock \emph{Journal of Machine Learning Research}, 13\penalty0
  (Mar):\penalty0 723--773, 2012.

\bibitem[Kantorovich(2006)]{kantorovich2006translocation}
Leonid~V Kantorovich.
\newblock On the translocation of masses.
\newblock \emph{Journal of Mathematical Sciences}, 133\penalty0 (4):\penalty0
  1381--1382, 2006.

\bibitem[Kuhn(1955)]{kuhn1955hungarian}
Harold~W Kuhn.
\newblock The hungarian method for the assignment problem.
\newblock \emph{Naval research logistics quarterly}, 2\penalty0 (1-2):\penalty0
  83--97, 1955.

\bibitem[Lavenant et~al.(2018)Lavenant, Claici, Chien, and
  Solomon]{lavenant2018dynamical}
Hugo Lavenant, Sebastian Claici, Edward Chien, and Justin Solomon.
\newblock Dynamical optimal transport on discrete surfaces.
\newblock \emph{ACM Transactions on Graphics (TOG)}, 37\penalty0 (6):\penalty0
  1--16, 2018.

\bibitem[Luise et~al.(2018)Luise, Rudi, Pontil, and
  Ciliberto]{luise2018differential}
Giulia Luise, Alessandro Rudi, Massimiliano Pontil, and Carlo Ciliberto.
\newblock Differential properties of sinkhorn approximation for learning with
  wasserstein distance.
\newblock In \emph{Advances in Neural Information Processing Systems}, pages
  5859--5870, 2018.

\bibitem[Luise et~al.(2019)Luise, Salzo, Pontil, and
  Ciliberto]{luise2019sinkhorn}
Giulia Luise, Saverio Salzo, Massimiliano Pontil, and Carlo Ciliberto.
\newblock Sinkhorn barycenters with free support via frank-wolfe algorithm.
\newblock In \emph{Advances in Neural Information Processing Systems}, pages
  9318--9329, 2019.

\bibitem[Monge(1781)]{monge1781memoire}
Gaspard Monge.
\newblock M{\'e}moire sur la th{\'e}orie des d{\'e}blais et des remblais.
\newblock \emph{Histoire de l'Acad{\'e}mie Royale des Sciences de Paris}, 1781.

\bibitem[Peyr{\'e} et~al.(2019)Peyr{\'e}, Cuturi,
  et~al.]{peyre2019computational}
Gabriel Peyr{\'e}, Marco Cuturi, et~al.
\newblock Computational optimal transport.
\newblock \emph{Foundations and Trends{\textregistered} in Machine Learning},
  11\penalty0 (5-6):\penalty0 355--607, 2019.

\bibitem[Rigollet and Weed(2018)]{rigollet2018entropic}
Philippe Rigollet and Jonathan Weed.
\newblock Entropic optimal transport is maximum-likelihood deconvolution.
\newblock \emph{Comptes Rendus Mathematique}, 356\penalty0 (11-12):\penalty0
  1228--1235, 2018.

\bibitem[Schiebinger et~al.(2019)Schiebinger, Shu, Tabaka, Cleary, Subramanian,
  Solomon, Gould, Liu, Lin, Berube, et~al.]{schiebinger2019optimal}
Geoffrey Schiebinger, Jian Shu, Marcin Tabaka, Brian Cleary, Vidya Subramanian,
  Aryeh Solomon, Joshua Gould, Siyan Liu, Stacie Lin, Peter Berube, et~al.
\newblock Optimal-transport analysis of single-cell gene expression identifies
  developmental trajectories in reprogramming.
\newblock \emph{Cell}, 176\penalty0 (4):\penalty0 928--943, 2019.

\bibitem[Shamir and Zhang(2012)]{shamir2012stochastic}
Ohad Shamir and Tong Zhang.
\newblock Stochastic gradient descent for non-smooth optimization: Convergence
  results and optimal averaging schemes.
\newblock \emph{CoRR}, abs/1212.1824, 2012.

\bibitem[Sinkhorn(1964)]{sinkhorn1964relationship}
Richard Sinkhorn.
\newblock A relationship between arbitrary positive matrices and doubly
  stochastic matrices.
\newblock \emph{The annals of mathematical statistics}, 35\penalty0
  (2):\penalty0 876--879, 1964.

\bibitem[Solomon et~al.(2015)Solomon, De~Goes, Peyr{\'e}, Cuturi, Butscher,
  Nguyen, Du, and Guibas]{solomon2015convolutional}
Justin Solomon, Fernando De~Goes, Gabriel Peyr{\'e}, Marco Cuturi, Adrian
  Butscher, Andy Nguyen, Tao Du, and Leonidas Guibas.
\newblock Convolutional wasserstein distances: Efficient optimal transportation
  on geometric domains.
\newblock \emph{ACM Transactions on Graphics (TOG)}, 34\penalty0 (4):\penalty0
  1--11, 2015.

\bibitem[Staib et~al.(2017)Staib, Claici, Solomon, and
  Jegelka]{staib2017parallel}
Matthew Staib, Sebastian Claici, Justin~M Solomon, and Stefanie Jegelka.
\newblock Parallel streaming wasserstein barycenters.
\newblock In \emph{Advances in Neural Information Processing Systems}, pages
  2647--2658, 2017.

\bibitem[Uppal et~al.(2019)Uppal, Singh, and Poczos]{uppal2019nonparametric}
Ananya Uppal, Shashank Singh, and Barnabas Poczos.
\newblock Nonparametric density estimation \& convergence rates for gans under
  besov ipm losses.
\newblock In \emph{Advances in Neural Information Processing Systems}, pages
  9086--9097, 2019.

\bibitem[Villani(2008)]{villani2008optimal}
C{\'e}dric Villani.
\newblock \emph{Optimal transport: old and new}, volume 338.
\newblock Springer Science \& Business Media, 2008.

\bibitem[Weed et~al.(2019)Weed, Bach, et~al.]{weed2019sharp}
Jonathan Weed, Francis Bach, et~al.
\newblock Sharp asymptotic and finite-sample rates of convergence of empirical
  measures in wasserstein distance.
\newblock \emph{Bernoulli}, 25\penalty0 (4A):\penalty0 2620--2648, 2019.

\bibitem[Zemel et~al.(2019)Zemel, Panaretos, et~al.]{zemel2019frechet}
Yoav Zemel, Victor~M Panaretos, et~al.
\newblock Fr{\'e}chet means and procrustes analysis in wasserstein space.
\newblock \emph{Bernoulli}, 25\penalty0 (2):\penalty0 932--976, 2019.

\end{thebibliography}
